\documentclass[10pt, conference, letterpaper]{ IEEEtran}
\IEEEoverridecommandlockouts

\usepackage{makecell}
\usepackage[utf8]{inputenc} 
\usepackage[T1]{fontenc}    
\usepackage{url}            
\usepackage{booktabs}       
\usepackage{amsfonts}       
\usepackage{nicefrac}       
\usepackage{microtype}      
\usepackage{xcolor}         
\usepackage{amsmath}
\usepackage{graphicx}
\usepackage{amssymb}
\usepackage{mathtools}
\usepackage{amsthm}
\usepackage{algorithm}
\usepackage{algorithmic}
\usepackage{comment}
\usepackage{cite}
\usepackage{textcomp}
\usepackage{dblfloatfix}

\usepackage[hidelinks]{hyperref} 
\def\BibTeX{{\rm B\kern-.05em{\sc i\kern-.025em b}\kern-.08em
    T\kern-.1667em\lower.7ex\hbox{E}\kern-.125emX}}
\begin{document}

\ifodd 1
\newcommand{\rev}[1]{\textcolor{blue}{#1}}
\newcommand{\revw}[1]{\textcolor{red}{#1}}
\newcommand{\revg}[1]{\textcolor{cyan}{#1}}
\newcommand{\revh}[1]{#1}
\newcommand{\com}[1]{\textbf{\color{red} \left(Comment: #1\right) }}
\newcommand{\comg}[1]{\textbf{\color{blue} \left(COMMENT: #1\right)}}
\newcommand{\response}[1]{\textbf{\color{blue} \left(RESPONSE: #1\right)}}
\else
\newcommand{\rev}[1]{#1}
\newcommand{\revh}[1]{#1}
\newcommand{\revw}[1]{#1}
\newcommand{\com}[1]{}
\newcommand{\comg}[1]{}
\newcommand{\response}[1]{}
\fi

\newtheorem{limitation}{Limitation}
\newtheorem*{question}{Main Question}
\newtheorem{remark}{Remark}
\newtheorem{assumption}{Assumption}
\newtheorem{theorem}{Theorem}
\newtheorem{definition}{Definition}
\newtheorem{proposition}{Proposition}
\newtheorem{lemma}{Lemma}

\def\p{\boldsymbol{p}}
\def\z{\boldsymbol{z}}
\def\h{\boldsymbol{h}}
\def\y{\boldsymbol{y}}
\def\V{\mathcal{V}}
\def\u{\boldsymbol{u}}
\def\v{\boldsymbol{v}}
\def\W{\boldsymbol{W}}
\def\g{\boldsymbol{g}}
\def\e{\boldsymbol{e}}
\def\thetab{\boldsymbol{\omega}}
\newcommand{\KL}{D_{\mathrm{KL}}}
\def\fhat{\hat{f}}
\def\J{\boldsymbol{J}}
\def\e{\boldsymbol{e}}
\def\pt{\tilde{\boldsymbol{p}}}
\def\r{\boldsymbol{r}}

\title{ZeroLock: Concurrent Memory-Efficient LLM  Training via Modular Update Decoupling}

\author{
    \IEEEauthorblockN{Wentao Dai$^*$,  Xuanran Li$^*$, Yuxiang Zhang, Ming Tang, Chao Huang}
    \thanks{Wentao Dai,   Yuxiang Zhang, and Ming Tang are with the Department of Computer Science and Engineering,
    Southern University of Science and Technology,
   China. Xuanran Li is with the Department of Mathematics,
    Southern University of Science and Technology,
   China. Chao Huang is with the School of Computing, Montclair State University, New Jersey, USA.
    Email: \{12311217,12312110,12410823\}@mail.sustech.edu.cn,  tangm3@sustech.edu.cn, huangch@montclair.edu. (Corresponding Author: Ming Tang)}
  \thanks{* Equal Contribution}
}

\maketitle

\begin{abstract}
Large language model (LLM) fine-tuning at the edge  adapts the model to scenario-specific data  while preserving privacy. Although existing studies proposed pipeline parallelism  to address the limited memory and computing resources of edge devices, they commonly rely on backpropagation (BP) training, which has a fundamental limitation of update locking and could experience severe throughput and memory bottlenecks. In this work, we propose a BP-free algorithm, called ZeroLock, that decouples the model updates into independent chunk updates by local objective construction. It breaks the update locking of BP and hence can improve throughput at the algorithm level and lower memory usage by reducing activation storage. To the best of our knowledge, we provide the first theoretical framework for such local objective construction-based approaches under general model chunk division by  mapping  local objectives to the global objective. We prove that ZeroLock has a convergence rate of $\tilde{\mathcal{O}}(1/\sqrt{T})$, which differs from BP only by polylogarithmic factors. We design a system for ZeroLock and build real-world prototypes, incorporating techniques such as early forwarding  and failure recovery for efficient and robust implementation. Experiments on the prototype show that compared to BP-based baselines, ZeroLock reduces the memory by  26.5\% and improves throughput by  4.9\%. 
\end{abstract}

\section{Introduction}
Edge intelligence deploys large language models (LLMs) at the network edge, achieving privacy-preserving and providing real-time inference services. In practice, it is important to fine-tune LLMs at the edge for two reasons. First, data distributions are usually scenario- and individual-specific, making it important to adapt the models to maintain inference accuracy. Second, data privacy concern and possibly frequent  model adaptation requirements make it difficult to upload the data to the cloud for fine-tuning. There are many such examples requiring both model adaptation and privacy preserving. For example, test-time training \cite{tandon2025end} adapts LLMs to maintain  user-specific memory for personalized service provision. Electroencephalogram signal-based tasks \cite{alghamdi2025cross} (e.g., emotion recognition, sleep staging) require model adaptation for individuals, as those signals demonstrate strong individual-dependent patterns. 

However, devices at the network edge usually have limited computation and memory resources, so it is difficult for a device or a GPU to fine-tune the entire LLM. To address this, existing studies have proposed \emph{pipeline parallelism} approaches \cite{huang2019gpipe,narayanan2019pipedream,confidant2025chen}. The main idea is to partition the model vertically into chunks of consecutive layers. The update of each model chunk corresponds to a \emph{stage} and is assigned to a different device, and these devices collaborate to update the chunks in a pipeline fashion. GPipe \cite{huang2019gpipe} is a typical approach and processes micro-batches sequentially.
1F1B  \cite{narayanan2019pipedream} interleaves forward pass and backward update  to improve pipeline utilization. Building upon 1F1B, PipeDream \cite{narayanan2019pipedream, narayanan2021pipedream2bw} enables asynchronous pipeline training via weight stashing. Recent studies \cite{fan2021dapple,narayanan2021megatron,li2021chimera,qi2024zerobubble} reduce bubbles in pipeline through placement search, virtual or bidirectional stages, and finer backward decomposition. Other studies \cite{lian2025ucp,gandhi2026moevement,wu2026pipemorph,chen2025collapipe,ryabinin2023swarm} focused on pipeline planning considering preemptions, stragglers, device heterogeneity, or communication efficiency.  Confidant \cite{confidant2025chen} implements 1F1B on smartphones.  

\begin{figure}
    \flushleft
    \includegraphics[height=1.8cm]{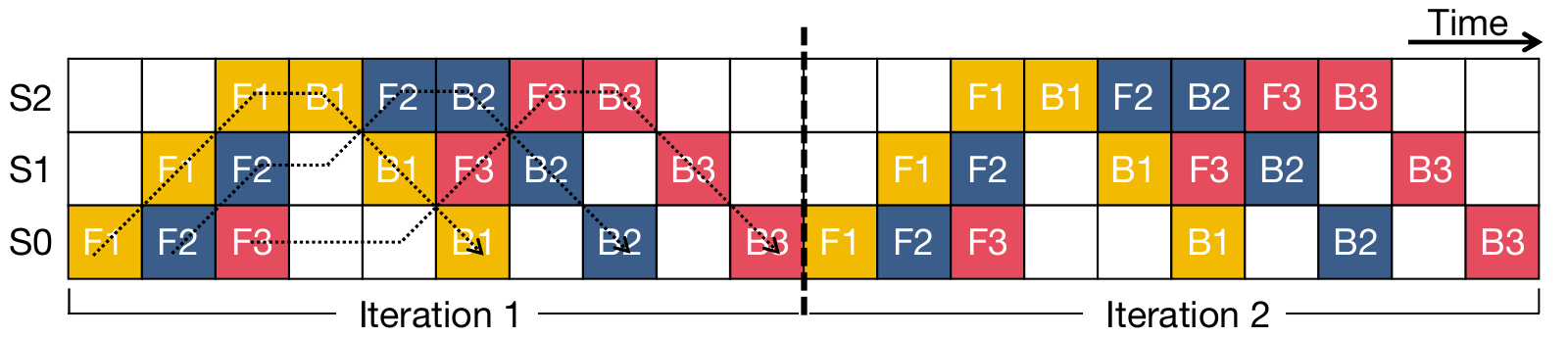}\vspace{-1.5mm}\\
    \centering(a)\vspace{-2mm}\\
    \flushleft\includegraphics[height=1.8cm]{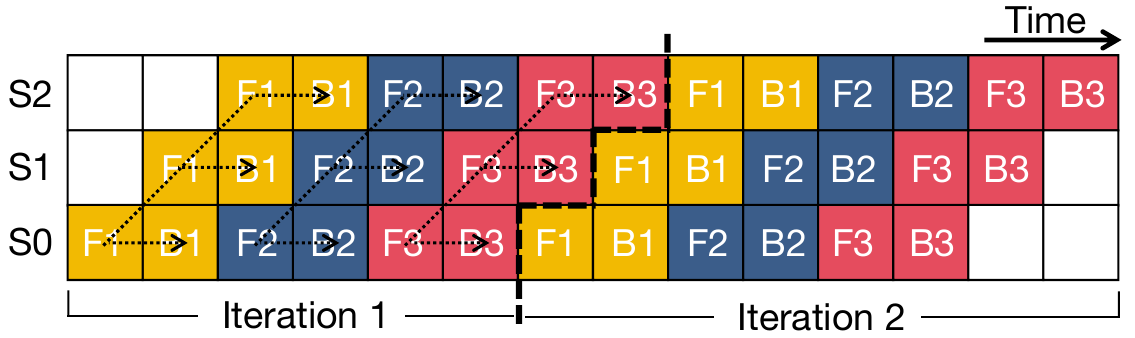}\vspace{-1.5mm}\\
    \centering (b)\\
    \includegraphics[height=1.9cm]{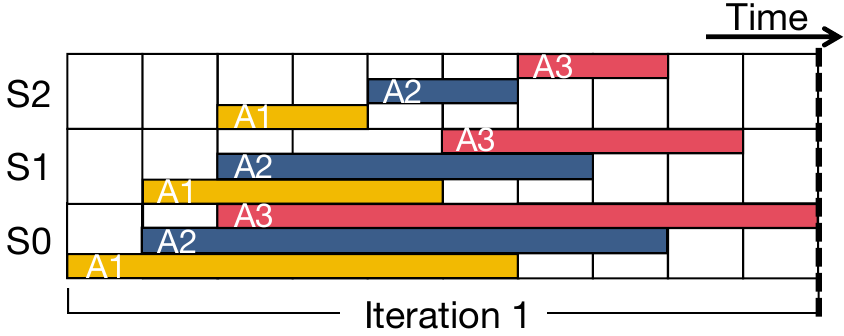}\includegraphics[height=1.9cm]{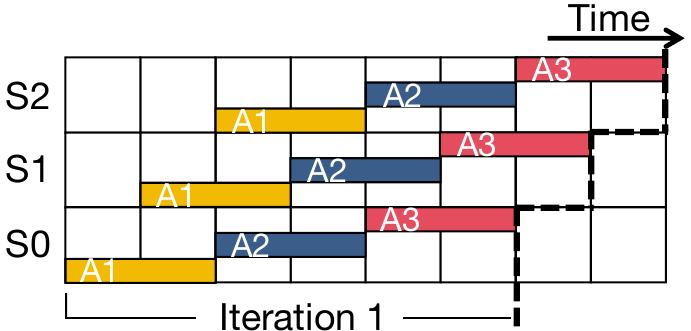}\vspace{-1.5mm}\\
    \quad(c)\qquad\qquad\qquad\qquad\quad\qquad(d)\vspace{-1mm}
    \caption{Throughput with (a) BP baseline  and (b) ZeroLock; memory with (c) BP baseline and (d) ZeroLock. In this figure, we use the typical 1F1B framework as the BP baseline, because many existing studies (e.g., \cite{confidant2025chen}) are built upon it. S0, S1, and S2 are stages, each corresponding to the updates of a model chunk and assigned to a device for execution. F$n$, B$n$, and A$n$ denote the forward pass, backward pass, and activations of the $n$-th microbathch, repetitively. The stripes in (c) and (d) indicate the duration that the corresponding activations are kept in memory. For example, 
    at time slot 3 of S0, A1-A3 are kept in memory in (c), while only A2 is kept in memory in (d).}
    \label{fig:framework}
\end{figure}

Most of these pipeline parallelism works (e.g.,  \cite{huang2019gpipe, narayanan2019pipedream, narayanan2021pipedream2bw, fan2021dapple,narayanan2021megatron,li2021chimera,qi2024zerobubble,lian2025ucp,gandhi2026moevement,wu2026pipemorph,confidant2025chen,chen2025collapipe,borzunov2023petals,ryabinin2023swarm}) rely on  backpropagation (BP) training and focus system-level scheduling optimization. \textbf{BP training has a fundamental limitation of update locking}\cite{pmlr-v80-huo18a,yedepth,rongguang2026survey}. That is, the model update contains a forward pass followed by a backward pass, so the update of upstream layers  needs to wait for the forward and backward computation of downstream layers. As a result, the approaches based on BP could have  crucial limitations:
\begin{itemize}
    \item \textbf{Throughput Bottleneck:} As shown in Fig. \ref{fig:framework} (a), due to  update locking, bubbles arise because an   upstream stage (e.g., S0) needs to wait for the backward pass of all downstream stages (e.g., S1, S2)  to accomplish its own backward pass. Although existing works (e.g., \cite{narayanan2019pipedream, narayanan2021pipedream2bw}) have proposed approaches to eliminate the bubbles, the updates of stages remain coupled due to the BP training, and hence the fundamental issue has not been resolved. 
    \item \textbf{Memory Bottleneck:} As shown in Fig. \ref{fig:framework} (c), due to  update locking,  an upstream stage (e.g., S0) needs to keep its activations  until all its downstream stages and itself (e.g., S0, S1, S2) accomplish their backward passes. This leads to significant memory waste for upstream stages.
\end{itemize}

The pervasive use of BP training and the resulting throughput and memory bottlenecks motivate the main question:
\begin{question}
   How can we design a system that fundamentally overcome the throughput and  memory bottlenecks at the algorithm level by breaking the update locking in BP?
\end{question}
BP-free training was proposed to break the update locking in BP. There are two categories. (I) Backward Gradient Estimation. For example, direct feedback alignment \cite{nokland2016direct,NEURIPS2019_f387624d} bypasses the chain rule by propagating target errors directly to all layers. Zero-order optimization \cite{NEURIPS2023_a6278101,wangcurvzo} computes perturbation-based loss differences during the forward pass to estimate gradients. However, this category usually leads to either significant model accuracy degradation for complex tasks or significant computing overhead. (II) Objective Reconstruction. For example, NoProp \cite{li2026noprop} reframes neural modules as independent denoising units that map noisy target embeddings back to clean targets. Predictive coding \cite{millidge2022predictive} alternates between minimizing chunk-wise prediction errors and updating weights using the local errors. {Local objective construction} (e.g., depth-progressive monotonic learning \cite{yedepth}) equips each layer with an independent local objective, allowing purely local gradient computation. Within category (II), \textbf{local objective construction  decouples the modular updates across stages} and has comparable model accuracy to BP approaches,  so it is a promising candidate to address the Main Question. Specifically, since each chunk is updated independently using its local loss, its update does not need to wait for the forward and backward passes  of its downstream chunks, so bubbles due to waiting can be removed; meanwhile, each chunk needs to keep its activations only until its own updates, reducing memory usage.

In this work, based on local objective construction-based BP-free approaches, we aim to propose a pipeline parallelism LLM training system to fundamentally overcome the throughput and memory bottlenecks. However, this design is not straightforward and needs to address the following questions:
\begin{itemize}
    \item[Q1] How to fine-tune LLM with local objective construction?
    \item[Q2] Does the modular decoupling of local objective construction theoretically harm the model convergence?
    \item[Q3] How to design and construct real-world prototype systems (for both multi-GPU server and Android phone scenarios) to achieve efficient and robust implementation?
\end{itemize}

Answering Q1 requires concrete design to incorporate LLM-specific characteristics (e.g., token sequence, low-rank adaptation for efficient fine-tuning) into local objective construction framework. Answering Q2 is very  challenging. This is because the modular update decoupling breaks global objective into local objectives of model chunks, while  recent studies lack frameworks to quantifying how optimizing the local objectives improves    the global objective. Answering Q3 is non-trivial due to the real-world engineering requirements of light-weight, high-throughput, and robust implementation. 

Existing works  tried to address one or two of Q1--Q3. PPLL \cite{guo2024ppll} places gradient-isolated vision blocks on different GPUs and transfers their features for block-wise updates. FluidPipe \cite{aljahdali2025fluidpipe} focuses on LLM and adds an auxiliary task head to the first part of a two-stage model, avoiding the local updates on the first part from waiting for the second part's gradient. SCPL \cite{ho2026scpl} decouples BP via per-segment supervised contrastive losses for synchronized multi-GPU model parallelism. However, these works \cite{guo2024ppll,aljahdali2025fluidpipe,ho2026scpl} considered hidden state alignment for Q1, failing to characterize task-specific information in local objective construction, and they did not address Q2. For Q3, although they propose high-level pipeline logic, their designs remain at a conceptual level, lacking critical implementation details such as continuous cross-batch pipelining, stage I/O queues, and explicit RPC primitives. Meanwhile, they did not provide deployable prototype for real-world execution. Although recent work LoPT \cite{shi2026lopt} addressed Q2, its analysis limits to two chunks and fails to provide a general analytical framework for connecting local objective to global objective.

We address those challenges and answer Q1--Q3. 
Our main contributions are summarized as follows:
\begin{itemize}
    \item ZeroLock Algorithm: To answer Q1, we propose a local objective construction-based BP-free algorithm, called ZeroLock, for LLM fine-tuning, which incorporates  low-rank adaptation (LoRA) \cite{hu2022lora} and LLM-compatible readout head and loss function. It breaks the update coupling across model chunks, and hence reduces the pipeline bubbles and activation storage at the algorithm level. 
    \item Theoretical Analysis: To answer Q2, we  establish theoretical equivalences to represent the local optimum of model chunks in a form of global objective. Then, the local updates of model chunks can be equivalently represented as global updates, with which the global convergence can be derived. To the best of our knowledge, this is the first analytical framework for the local objective construction algorithms under general model chunk division, which provides a systematic approach for analyzing global convergence given  decoupled local updates. We prove that ZeroLock has a convergence rate of $\tilde{\mathcal{O}}(1/\sqrt{T})$,  which differs from BP by only polylogarithmic factors.
    \item Real-World System Design: To answer Q3, based on ZeroLock, we design  systems and build prototypes for both multi-GPU server and Android phone scenarios. This system can achieve concurrent throughput (see Fig. \ref{fig:framework} (b)) and lower  memory usage by reducing activation storage (see Fig. \ref{fig:framework} (d)). Techniques such as early forwarding  and failure recovery are proposed to ensure light-weight, high-throughput, and robust implementation. To the best of our knowledge, we build the first  prototype on  Android system for local objective construction-based BP-free algorithm.
    \item Evaluation on Real-World Prototype: Experiments on multi-GPU server show that when compared with BP-based  baselines,  ZeroLock system reduces the memory usage by 26.5\% and  improves the  throughput by   4.9\%. The fine-tuning of TinyLlama on Android phones experiences a peak PSS of less than 4000 MiB,  a battery temperature of around 37$^\circ$C, and a wall-clock time of 1644.1 seconds, indicating that the implementation is practically feasible. Our code is available at 
\url{https://anonymous.4open.science/r/unlock_trainer-105B}.
\end{itemize}
The rest of this paper is organized as follows. Sections \ref{sec:algorithm} and \ref{sec:system} present ZeroLock algorithm and its system design, respectively. Experiments are  in Section \ref{sec:experiments}. Section \ref{sec:conclusion} concludes this work. 

\section{ZeroLock Algorithm and  Analysis}\label{sec:algorithm}


\subsection{ZeroLock Algorithm}\label{subsec:unlok}

Consider an LLM with an embedding operation $E(x)$ and a set of transformer layers $\mathcal{L}$. We introduce LoRA \cite{hu2022lora} to reduce the   trainable parameters in fine-tuning. Let $w_l\in\mathbb{R}^{d_1\times d_2}$ denote the base weights of the $l$-th transformer layer, which is frozen during fine-tuning. Let $A_l\in\mathbb{R}^{d_3\times d_2}$ and $B_l\in\mathbb{R}^{d_1\times d_3}$ denote the trainable low-rank decomposition matrices in LoRA. Then, the parameter updates of  transformer layer $l\in\mathcal{L}$ is represented by $\Delta w_l = B_lA_l$, so the  weights after fine-tuning are $w_l+\Delta w_l$. To achieve modular update decoupling in LLM fine-tuning, we partition the entire model into model chunks and introduce local objectives to these chunks for independent updates.\footnote{The main idea is inspired by the local objective construction approach in  \cite{yedepth}. Different from \cite{yedepth} on classification tasks,  ZeroLock  handles LLM fine-tuning by incorporating LoRA and LLM-compatible head and loss design.}

\textbf{Chunk Partitioning:} Consider an LLM split into an embedding operation  $E(\cdot)$ and  $K$  chunks, where  chunk $k$ contains one or multiple consecutive transformer layers, denoted by set $\mathcal{L}_k$. Let $W_k=(w_l~|~ l\in \mathcal{L}_k)$ and $\Delta W_k=(\Delta w_l~|~l\in\mathcal{L}_k)$ denote the frozen base weights and the trainable LoRA parameters in chunk $k$,  respectively, leading to a mapping $f_k(\cdot ; W_k + \Delta W_k)$ from the chunk input to output. Let $x$ be the LLM input. Define $\h_0\in\mathbb{R}^{S\times d}$ and $\h_k\in\mathbb{R}^{S\times d}$ as the token embedding and the output of chunk $k$, respectively, with $S$ as the sequence length:
\begin{equation}\label{eq:chunk-split}
  \h_0 = E(x), ~
  \h_k = f_k(\h_{k-1}; W_k +\Delta W_k),
  k = 1,\cdots,K.
\end{equation}

\textbf{Readout Head:} To enable local updates, we introduce a {frozen} local readout head after chunk $k=1,\cdots, K$: 
\begin{equation}\label{eq:local-head}
  \z_k = \mathrm{Norm}(\h_k) W_{\mathrm{lm}}^\top, ~\p_{k,i} = \mathrm{softmax}(\z_{k,i}).
\end{equation}
In \eqref{eq:local-head}, $W_{\mathrm{lm}}$ is the head weight that maps the normalized hidden representation $\mathrm{Norm}(\h_k)$ to vocabulary logits $\z_k\in\mathbb{R}^{S\times V}$, with $V$ being the vocabulary size. Vector $\z_{k,i}$ is the $i$-th row of $\z_k\in\mathbb{R}^{S\times V}$. With $\mathrm{softmax}(\cdot)$,  $ \p_{k,i}\in\mathbb{R}^{1\times V}$ is the predictive distribution of next token  over the  vocabulary space.  


\textbf{Local Objective:} With the output of the readout head, the local loss of each chunk $k$ consists of a task-dependent term $L_{\mathrm{Task}}(\cdot)$ and a consistency term $L_{\mathrm{Consis}}^{k \rightarrow k-1}(\cdot)$:
\begin{equation}\label{eq:belief-kd}
   L_k(\p_k) = \alpha L_{\mathrm{Task}}(\p_k)
     + (1-\alpha) L_{\mathrm{Consis}}^{k \rightarrow k-1}(\p_k), 
\end{equation}
where $\alpha\in(0,1]$ is a weight, and $\p_k=(\p_{k,i}, i\in\Omega)$ with $\Omega$ as the token sequence removing prompt or padding positions. 

(i) Task-dependent term $L_{\mathrm{Task}}$ is a local version of the global objective. It aligns the output of the readout head  of each chunk $k$ to the global ground-truth target:
\begin{equation}\label{eq:task-loss}
 L_{\mathrm{Task}}(\p_k) = -\frac{1}{|\Omega|} \sum_{i \in \Omega}
      D_{\psi}\!\left( \p_{k,i}, \p_y\right),
\end{equation}
where $\p_y\in\mathcal{R}^{1\times V}$ is the one-hot target. Let $\mathcal{P}$ denote the space of $\p$. $D_{\psi}(\cdot)$ is the Bregman divergence induced by a strictly convex and differentiable potential function $\psi: \mathcal{P} \rightarrow \mathbb{R}$, i.e., $D_{\psi}(\u, \v) = \psi(\u) - \psi(\v) - \langle \nabla \psi(\v), \u - \v \rangle$. KL divergence is a specific example of  Bregman divergence by selecting negative entropy as the  potential function. Importantly, since minimizing KL divergence is equivalent to minimizing cross-entropy, \eqref{eq:task-loss} can be replaced with cross-entropy loss in practice.

(ii) Consistency term  aligns the outputs of  chunk $k$ to those of chunk $k-1$, ensuring coherent outputs across chunks.
\begin{equation}\label{eq:ktos}
  L_{\mathrm{Consis}}^{k \rightarrow k-1}(\p_k) = \frac{1}{|\Omega|}
     \sum_{i \in \Omega}
     D_{\psi}\!\left( \p_{k,i}, \mathrm{sg}(\p_{k-1,i})\right),
\end{equation}
where $\mathrm{sg}(\cdot)$ is stop-gradient operator with $\mathrm{sg}(\p)=\p$ and $\nabla\mathrm{sg}(\p)=\boldsymbol{0}$. Similarly, KL divergence can be used in practice. 

\textbf{Fine-Tuning Process:} This process contains $T$ iterations. In each iteration, the low rank matrices $A_l$ and $B_l$ of  layer $l\in\mathcal{L}_k$ are fine-tuned  using stochastic gradient descents:
\begin{equation}\label{eq:AB}
   A_l \leftarrow  A_l - \eta \frac{\partial f_k}{\partial A_l}, ~ B_l \leftarrow B_l -  \eta \frac{\partial f_k}{\partial B_l}, 
\end{equation}
where $\eta$ is the learning rate.  
\begin{remark}
    According to \eqref{eq:AB}, layers within the same chunk are updated using BP, while the updates of layers from different chunks are decoupled and can be executed in parallel. This removes the parallelism bubbles at the algorithm level (see Fig. \ref{fig:framework} (b)) and avoids the need for keeping the activations during the downstream's updates  (see Fig. \ref{fig:framework} (d)).
\end{remark}

\subsection{Theoretical Analysis: Chunk-Wise Performance}\label{subsec:layer}
With chunk-wise performance analysis, we aim to provide insights into how the global objective value changes across model chunks. We focus on one token in the sequence and omit the token subscript. 
Recall that $\p$ is the readout head output of a chunk and is a distribution over vocabulary space. Let $L: \mathcal{P} \rightarrow \mathbb{R}$ denote the global objective given $\p\in\mathcal{P}$, i.e., $L(\p)\triangleq L_{\mathrm{Task}}(\p)$, which is the loss function characterizing how much $\p$ deviates from global ground-truth target. Let $\p_{k}^*$ denote the optimal $\p_{k}$ of chunk $k$ that minimizes the local loss $L_k(\cdot)$, i.e., $\p_{k}^* \triangleq \arg \min_{\p\in \mathcal{P}} L_k(\p)$:  
\begin{lemma}[Local Optimum]\label{lem:p-opt}
   The local optimum $\p_{k}^*$ can be equivalently represented by
\begin{equation}\label{eq:pstar}
    \p_{k}^* \!=\! \arg\min_{\p \in \mathcal{P}} \left\{ \langle \nabla L(\p_{k\!-\!1}), \p \!-\! \p_{k\!-\!1} \rangle \!+\! \frac{1}{\alpha} D_{\psi}(\p, \p_{k\!-\!1}) \right\}.
\end{equation}
\end{lemma}
\begin{proof}
    Recall the  definition of the Bregman divergence,
\begin{equation}\label{eq:Dpsi}
D_\psi(\u,\v)=\psi(\u)-\psi(\v)-\langle\nabla\psi(\v),\u-\v\rangle.
\end{equation}
Substituting $(\u=\p,\v = \p_{k-1})$  and   $(\u=\p,\v = \p_y)$ into \eqref{eq:Dpsi} respectively and rearranging the former, we have 
$\psi(\p)\!=\!\psi(\p_{k\!-\!1})\!+\!\langle\nabla\psi(\p_{k\!-\!1})\!,\!\p-\p_{k\!-\!1}\rangle\!+\!D_\psi(\p,\p_{k\!-\!1})$, $D_\psi(\p,\p_y)=\psi(\p)-\psi(\p_y)-\langle\nabla\psi(\p_y),\p-\p_y\rangle$. Thus,
\begin{align}\label{eq:D3}
    D_\psi(\p,\p_y) 
   {=} &    \psi(\p_{k-1})-\psi(\p_y)-\langle\nabla\psi(\p_y),\p_{k-1}-\p_y\rangle \notag\\
    &  -\langle\nabla\psi(\p_y),\p-\p_y\rangle +\langle\nabla\psi(\p_y),\p_{k-1}-\p_y\rangle\notag\\
    &   +\langle\nabla\psi(\p_{k-1}),\p-\p_{k-1}\rangle +D_\psi(\p,\p_{k-1})\notag\\
    \overset{(a)}{=} &
     D_\psi(\p_{k-1},\p_y) +\langle\nabla L(\p_{k-1}),\p - \p_{k-1} \rangle \notag\\
    & +D_\psi(\p,\p_{k-1}),
\end{align}
where  (a) holds as the definition of $D_\psi(\p_{k-1},\p_y)$ and  $L(\p_{i-1})$. 

Substituting \eqref{eq:D3} into the local loss function $L_k(\p)$ leads to $L_k(\p) = \alpha D_\psi(\p_{k-1},\p_y) + \alpha  \langle\nabla L(\p_{k-1}),\p - \p_{k-1} \rangle + D_\psi(\p,\p_{k-1})$. 
Given $\p_{k-1}$, term $\alpha D_\psi(\p_{k-1},\p_y)$ is a constant, so it can be omitted in optimization. Proof completes. 
\end{proof}
Lemma \ref{lem:p-opt} represents the local optimum $\p^{\star}$ in a  form of the global objective $L(\cdot)$, connecting local and global objectives. Meanwhile, it demonstrates a balance between moving toward the steepest descent direction of global objective $L(\p)$ and penalizing  deviations from the previous distribution $\p_{k-1}$.

As in many existing convergence analysis on BP or BP-free (e.g., \cite{han2024convergence, shi2026lopt}), we assume the smoothness of global objective. 
\begin{assumption}[Global Objective Smoothness]
\label{ass:bregman-smoothness}
The global objective $L(\cdot)$ is $\beta$-smooth relative to $\psi$ in Bregman geometry:
\begin{equation}
L(\u)\le L(\v) + \langle \nabla L(\v), \u-\v\rangle + \beta D_{\psi}(\u,\v), ~\u,\v\in \mathcal{P}.
\end{equation}
\end{assumption}

Then, we derive the chunk-wise performance of ZeroLock.
\begin{proposition}[Chunk-Wise Performance]
\label{thm:perturbed-descent}
Suppose $\p_{k-1}$, $\p_k^\star$, and $\p_k$ lie in a compact set of the relative interior of the simplex, and \(\psi(\cdot)\) is locally strongly convex. Define $\delta_k \triangleq D_{\psi}(\p_k, \p_k^*)$. Under Assumption \ref{ass:bregman-smoothness} and $\alpha < \frac{1}{\beta}$,
\begin{equation}
\label{equ:telescoping-bound}
    L(\p_K) \!\le\! L(\p_0) \!-\! \sum_{k=1}^K \left(\frac{1}{\alpha}\!-\!\beta \right)D_{\psi}(\p_k, \p_{k\!-\!1}) \!+\! \frac1\alpha \sum_{k=1}^K \delta_k.
\end{equation}
\end{proposition}
\begin{proof}
Substituting $\u=\p_k$ and $\v = \p_{k-1}$ in Assumption \ref{ass:bregman-smoothness}, 
\begin{multline}\label{eq:f-proof1}
L(\p_k) \le  \underbrace{\langle \nabla L(\p_{k-1}),\p_k-\p_k^*\rangle}_{(i)}  + \underbrace{\langle \nabla L(\p_{k-1}),\p_k^*-\p_{k-1}\rangle}_{(ii)}  \\
+ L(\p_{k-1}) + \beta D_{\psi}(\p_k,\p_{k-1}).
\end{multline}

 Since $\p_k^\star$ lies in the relative interior of the probability simplex, the KKT condition of (\ref{eq:pstar}) gives $\nabla L(\p_{k-1})+\frac{1}{\alpha}\bigl(\nabla\psi(\p_k^\star)-\nabla\psi(\p_{k-1})\bigr)+\lambda_k\mathbf 1 = 0$.
Accordingly, term (i) satisfies
\begin{align}
&\alpha \langle \nabla L(\p_{k-1}),\p_k-\p_k^\star\rangle \notag\\ 
\overset{(a)}{=} & \langle \nabla\psi(\p_{k-1})-\nabla\psi(\p_k^\star),\p_k-\p_k^\star\rangle  \notag\\ 
\overset{(b)}{=}& \delta_k+D_{\psi}(\p_k^\star,\p_{k-1})-D_{\psi}(\p_k,\p_{k-1}).
\end{align}
Here, (a) holds due to the  KKT condition and  $\mathbf 1^\top(\p_k-\p_k^\star)=0$, where the latter holds since $\p_k-\p_k^\star$ lies in the simplex tangent space. (b) holds based on Bregman three-point identity.
 According to the definition of $\p_k^*$, term (ii) satisfies 
$\langle \g_k,\p_k^*-\p_{k-1}\rangle + \frac{1}{\alpha}D_{\psi}(\p_k^*,\p_{k-1}) \le 0$. 
Substituting (i) and (ii) into \eqref{eq:f-proof1}
and summing  over $k=1,\dots,K$  complete the proof. 
\end{proof}

Proposition \ref{thm:perturbed-descent} shows that the global loss $L(\p_K)$ of the final chunk tends to decrease as the number of chunks increases, with a bounded error $\sum_{k=1}^{K}\delta_k/\alpha$ resulting from the fine-tuning suboptimality of model chunks due to finite parameters.\footnote{Although a relevant analysis on chunk-wise performance is presented in  \cite{yedepth}, our Proposition \ref{thm:perturbed-descent} generalizes that in \cite{yedepth} (i) from KL divergence to general Bregman divergence and (ii) by explicitly deriving the specific form of the error $\delta_k$ (rather than simply assuming an error).}




\subsection{Theoretical Analysis: Algorithm Convergence}\label{subsec:convg}
We provide the first analytical framework for local objective construction-based BP-free approach under general chunk division. As mentioned earlier, the major challenge comes from characterizing global convergence based on the local updates of each model chunk. To overcome this, we equivalently map local updates to global updates (Lemma \ref{lem:update}); based on this, we prove that the convergence rate of ZeroLock is $\tilde{\mathcal{O}}(1/\sqrt{T})$ (Theorem \ref{thm:stationarity}), differing from $\mathcal{O}(1/\sqrt{T})$ of BP by a polylogarithmic factor.\footnote{$f(T)= \tilde{\mathcal{O}}(g(T))$ means $f(T)=\mathcal{O}(g(T)\log^c T)$ for some constant $c$. The extra $\log T$ is a slow-varying factor  that becomes negligible for large $T$.}

For iteration $t$, let $\thetab^t = (\thetab_k^t,k=0,1,\cdots, K)$ denote the parameters of chunks to be updated. We introduce this notation to generalize the analysis for various scenarios (e.g., LLM, CNN). Following the notation  in Section \ref{subsec:unlok}, $\thetab_0^t$ is the fixed parameters of operator $E(\cdot)$,  and $\thetab_k^t\triangleq \Delta W_k^t + W_k$ for chunk $k$. Let $\p^t_0, \p^t_1, \dots, \p^t_K$ denote the readout head output of the  associated model chunks. For analytical simplicity, we set the Bregman divergence in \eqref{eq:task-loss} and \eqref{eq:ktos} as KL divergence.  Then, the local loss is equivalent  to 
\begin{equation}
L_k^t(\p;\alpha)=\alpha \KL(\p~\|~\p_y) + (1-\alpha) \KL(\p~\|~\p_{k-1}^t).
\end{equation}


\textbf{Mapping from Local to Global Updates:} First, we show the equivalence on local objective as follows. 
\begin{lemma}[Objective Equivalence]\label{lem:equivalence}
    Minimizing $L_k^t(\p;\alpha)$ is equivalent to minimizing $ \KL\left(\p~\|~\p_k^{*,t}(\alpha)\right)$, where 
\begin{equation}\label{eq:pstart}
    \p_k^{*,t}(l; \alpha) = \frac{\p_y(l)^\alpha\left(\p_{k-1}^t(l)\right)^{1-\alpha}}{\sum_{l'=1}^m \p_y(l')^\alpha \left(\p_{k-1}^t(l')\right)^{1-\alpha}},
\end{equation}
with $l$ denoting the $l$-th element of the associated vector and $m$ denoting the length of the vector.
\end{lemma}
\begin{proof}
Local loss $L_k^t(\p;\alpha)$ can be equivalently represented as
\begin{equation}
L_k^t(\p;\alpha)=\sum_{l=1}^m \p(l)\log\frac{\p(l)}{\p_y(l)^\alpha\left(\p_k^t(l)\right)^{1-\alpha}}.
\end{equation}
Define $Z_k^t(\alpha)=\sum_{l=1}^m \p_y(l)^\alpha\left(\p_{k-1}^t(l)\right)^{1-\alpha}$. Then, 
\begin{equation}\label{eq:kl-eq}
L_k^t(\p;\alpha)=\KL(\p~\|~\p_k^{*,t}(\alpha))-\log Z_k^t(\alpha).
\end{equation}
Note that for any chunk $k$, vectors $\p_y$ and $\p_{k-1}^t$ are given, so $Z_k^t(\alpha)$ is a constant. Hence, this lemma is proved. 
\end{proof}
Note that $\p_k^{\ast,t}(\alpha)\triangleq(\p_k^{\ast,t}(l;\alpha), l=1,\cdots,m)$ is essentially the optimal solution that minimizes $L_k^t(\p;\alpha)$, with proof omitted here. Thus,   minimizing $ \KL\left(\p~\|~\p_k^{*,t}(\alpha)\right)$ is equivalent to  minimizing  the difference between $\p$ and the optimal solution. 

Then, we can represent global updates in ZeroLock as a form of local updates. We define $\p_k^t=\fhat_k(\p_{k-1}^t;\thetab_k^t)$ as the mapping from the readout head output $\p_{k-1}^t$ of chunk $k$ to $\p_k^t$, given parameter $\thetab_k^t$ of chunk $k$. Although this mapping cannot be directly obtained using the readout head operation (as it is not a one-to-one correspondence), it can be approximated given the statistics of $\p_k^t$ across chunks. Also, this mapping is introduced for theoretical analysis, and it does not need to be obtained in practice. Define the Jacobian matrix  $\J_{\theta_k,k}^t \triangleq \partial \fhat_k(\p_{k-1}^t;\thetab_k^t)/\partial \thetab_k^t$. The stop-gradient operation induces the following Jacobian matrix for all parameters:
\begin{equation}
   \J^t= \operatorname{blkdiag}(\J_{\thetab_1,1}^t, \J_{\thetab_2,2}^t,\cdots, \J_{\thetab_K,K}^t),
\end{equation}
where $\operatorname{blkdiag}(\cdot)$ denotes block-diagonal matrix with all elements outside the diagonal blocks being zero. Define $\e_k^t(\alpha)=\log \p_k^t-\log \p_k^{\ast,t}(\alpha)$, which is the gap between the recent $\p_k^t$ and the optimal $ \p_k^{\ast,t}(\alpha)$ in logarithmic form. 

Based on Lemma \ref{lem:equivalence} and $\p_k^t=\fhat_k(\p_{k-1}^t;\thetab_k^t)$, minimizing $L_k^t(\p;\alpha)$ for all chunks over  dataset $\mathcal{D}$  is equivalent to finding the parameter $\thetab=(\thetab_k, k=0,1,\dots, K)$ that minimizes
\begin{equation}\label{eq:R}
\!\!\!\mathcal R_t(\thetab;\alpha)\!=\!\mathbb E_{(x,y)\sim\mathcal D}\!\!\left[\sum_{k=1}^K\! \KL\!\!\left(\fhat_k(\p_{k\!-\!1}^t;\thetab_k)\|\p_k^{\ast,t}(\alpha)\!\right)\!\right]\!.
\end{equation} 
\begin{lemma}[Global Update Equivalence]\label{lem:update}
    Based on Lemma \ref{lem:equivalence}, the update rule of  ZeroLock can be represented as follows:\footnote{The algorithm in \cite{yedepth} differs from ZeroLock by maintaining an encoder and having its output $c(x)$ as the input for each chunk. To generalize it, we can append a column of  $\J^t_{\thetab_{c},k}$ for $k=1,\cdots,K$ before the columns in $\J^t$, where $\J^t_{\thetab_{c},k}$ is the Jacobian matrix of $f_k^t(\cdot)$ with respect to the parameters $\thetab_{c}$ of the encoder. Under this modification, Lemma \ref{lem:update} still holds.}
\begin{equation}
\label{equ:sid-jacobian-update}
\thetab^{t+1} \leftarrow \thetab^t - \eta_t \nabla_{\thetab} \mathcal R_t(\thetab^t;\alpha),
\end{equation}
where $\nabla_{\thetab} \mathcal R_t(\thetab;\alpha) = (\J^t)^\top \e^t$,  $\e^t\triangleq(\e^t_k, k=0,1,\cdots, K)$.
\end{lemma}
\begin{proof}
According to \eqref{eq:kl-eq} in Lemma \ref{lem:equivalence}, given output $\p_k^t$,
\begin{equation}\label{eq:L1}
\nabla_{\p_k^t} L_k^t(\p_k^t;\alpha) = \log \p_k^t-\log \p_k^{\ast,t}(\alpha)+\mathbf 1.
\end{equation}
We apply the chain rule. Let $\ell_k^t(\thetab_k^t) \triangleq  L_k^t(\p_k^t)$, and we omit  $\alpha$ for presentation simplicity. For the \(u\)-th coordinate of \(\thetab_k\),
\begin{equation}
    \frac{\partial \ell_k^t(\thetab_k^t)}{\partial\thetab_{k,u}^t} = \sum_{l=1}^m \frac{\partial L_k^t(\p_k^t)}{\partial \p_k^t(l)}  \frac{\partial \p_k^t(l)}{\partial\thetab_{k,u}^t} \overset{(a)}{=} \sum_{l=1}^m \e_k^t(l) \frac{\partial \p_k^t(l)}{\partial\thetab_{k,u}}.
\end{equation}
Equality (a) holds because $\sum_{l=1}^m \frac{\partial \p_k^t(l)}{\partial\thetab_u^t} = \frac{\partial}{\partial\thetab_u^t} \sum_{l=1}^m \p_k^t(l) =  0$,
since  $\sum_{l=1}^m \p_k^t(l)=1$, and \eqref{eq:L1}. Finally, stacking all coordinates and the gradients of all chunks yields this lemma. 
\end{proof}


\textbf{Convergence:}  Define  $\|\boldsymbol a\|_{\mathcal D}^2 \triangleq \mathbb E_{x\sim\mathcal D_X}[\sum_{i=1}^{I}\|a_i(x)\|_2^2]$, $\langle \boldsymbol a,\boldsymbol b\rangle_{\mathcal D} \triangleq \mathbb E_{x\sim\mathcal D_X}[\sum_{i=1}^{I}\langle a_i(x),b_i(x)\rangle]$. Here, $I$ is the size of $\boldsymbol{a}$ and $\boldsymbol{b}$,  $x$ is the input of the model, and $\mathcal{D}_X$ is the distribution of $x$ in dataset $\mathcal{D}$. Let $\p(\thetab) = (\p_k(\thetab)\triangleq\fhat_k(\p_{k-1};\thetab_k), k=0, \cdots, K)$ denote  the readout head output given  $\thetab$. We rewrite $\mathcal{R}_t(\thetab;\alpha) = \mathcal{R}(\thetab;\alpha, \pt^t)$ where $\p^t=\p(\thetab^t)$  for presentation simplicity.  
The following Assumptions \ref{assump:bounded-update-moment}--\ref{assump:refresh-smoothness}  are commonly considered in existing   analysis (e.g., \cite{shi2026lopt,han2024convergence}).  Assumption \ref{ass:stability} is reasonable given  the bounded gradient in Assumption \ref{assump:bounded-update-moment} and the probability simplex space of $\pt^t$.
Assumption \ref{ass:sid-minibatch} is reasonable under simple random sampling and given the bounded gradient.

\begin{assumption}[Bounded Gradient]
\label{assump:bounded-update-moment} The gradient is upper bounded, i.e., $\mathbb E[\|\nabla_{\thetab} \mathcal{R}_t(\thetab;\alpha)\|^2]\le M(\alpha)$.
\end{assumption}

\begin{assumption}[Lipschitz Continuity]
\label{assump:teacher-lipschitz}
Readout output $\p(\thetab)$ is Lipschitz continuous with respect to $\thetab$, i.e.,  $\| \p(\thetab)- \p(\thetab')\|_{\mathcal D}\le \rho \|\thetab- \thetab'\|$, for all $\thetab$ and $\thetab'$ in parameter space.
\end{assumption}

\begin{assumption}[Smoothness]
\label{assump:refresh-smoothness}
For any $\thetab$ in parameter space, $\mathcal{R}_t(\thetab;\alpha)$ is $\beta_{R}(\alpha)$-smooth, i.e., $
    \mathcal R_t( \thetab';\alpha)\le \mathcal R_t( \thetab;\alpha)+\langle \nabla \mathcal R_t(\thetab;\alpha),\thetab'-\thetab \rangle
    +\frac{\beta_{R} (\alpha)}{2}\|\thetab'-\thetab\|^2$.
\end{assumption}
\begin{assumption}[Stability]\label{ass:stability}
The  first-order objective variation is quadratically controlled along the distribution update, i.e., $\left[\left\langle \nabla_{\pt^{t}}\mathcal R(\thetab^{t};\alpha,\pt^{t}),\pt^{t+1}-\pt^t\right\rangle_{\mathcal D} \right]_+ 
    \le C_{\rm fo}(\alpha)\|\pt^{t+1}-\pt^t\|_{\mathcal D}^2$.
\end{assumption}

\begin{assumption}[Sampling]
\label{ass:sid-minibatch}
The mini-batch sampling is unbiased and has a bounded variance, i.e., $||\nabla_{\thetab} \hat{\mathcal R}_t(\thetab;\alpha) - \nabla_{\thetab} \mathcal R_t(\thetab;\alpha)||^2 \leq \sigma^2(\alpha)/B$, where $\nabla_{\thetab} \hat{\mathcal R}_t(\thetab;\alpha)$ is the gradient under sampled mini-batch and $B$ is the mini-batch size. 
\end{assumption}

Define drift $\delta_t^{\mathrm{ref}}(\alpha)\triangleq \left[\mathcal R_{t+1}(\thetab^{t+1};\alpha)-\mathcal R_t(\thetab^{t+1};\alpha)\right]_+$. This drift exists as distribution $\p^t_{k-1}$ from chunk $k-1$ changes, and characterizes the change of $\mathcal{R}_t(\cdot)$ across iterations.
\begin{lemma}[Bounded Drift]
\label{prop:bounded-drift}
The cumulative drift of $\mathcal R_t(\cdot)$ is upper bounded, i.e.,   
    $\mathfrak D_T(\alpha) \triangleq \sum_{t=0}^{T-1}\mathbb E[\delta_t^{\mathrm{ref}}(\alpha)] \le (C_{\rm fo}(\alpha)+\beta_{R}(\alpha)/2)\rho^2M(\alpha)\sum_{t=0}^{T-1}\eta_t^2$.
\end{lemma}
\begin{proof}
According to the definition of $\delta_t^{\mathrm{ref}}(\alpha)$, by substituting $\thetab'=\thetab^{t+1}$ and $\thetab=\thetab_t$ in Assumption \ref{assump:refresh-smoothness}, $\delta_t^{\rm ref}(\alpha)\le (C_{\rm fo}(\alpha)+\frac{\beta_{R}(\alpha)}{2})\|\pt^{t+1}-\pt^t\|_{\mathcal D}^2$.
Based on Lipschitz continuity of $\pt^t\triangleq\pt(\thetab^t)$ in Assumption~\ref{assump:teacher-lipschitz}, $\|\pt^{t+1}-\pt^t\|_{\mathcal D}^2\le \rho^2\|\thetab^{t+1}-\thetab^t\|^2$. Let $ C_{\rm ref}\triangleq(C_{\rm fo}(\alpha)+\beta_{R}(\alpha)/2)\rho^2$. Based on Lemma \ref{lem:update} and Assumption~\ref{assump:bounded-update-moment},
$\mathbb E[\delta_t^{\rm ref}]  \leq C_{\rm ref}(\alpha) \mathbb E[\|\thetab^{t+1}-\thetab^t\|^2] 
   =C_{\rm ref}(\alpha)\eta_t^2\mathbb E[\| \nabla_{\thetab} \mathcal R_t(\thetab;\alpha)\|^2]\le C_{\rm ref}(\alpha) M(\alpha)\eta_t^2$. 
Summing over $t=0,\ldots,T-1$ completes the proof. 
\end{proof}

\begin{theorem}[Convergence]
\label{thm:stationarity}
Let $\eta_t= \eta_0/\sqrt{t+\gamma} \le 1/\beta_R(\alpha)$, where $\gamma>0$. Based on Lemma \ref{prop:bounded-drift} and Assumptions~\ref{assump:refresh-smoothness}--\ref{ass:sid-minibatch},
\begin{multline}
    \frac{1}{T} \sum_{t=0}^{T-1}\mathbb E\|\nabla_{\thetab} \mathcal R_t(\thetab;\alpha)\|^2 \le \frac{2(\mathcal R_0-\mathcal R_{opt}+\mathfrak D_T(\alpha))}{T \eta_0/\sqrt{T+\gamma}}\\
    + \frac{\beta_R(\alpha) \sigma(\alpha)^2 \sum_{t=0}^{T-1}\eta_t^2/B}{T \eta_0/\sqrt{T+\gamma}} 
    = \tilde{\mathcal{O}}\left(\frac{1}{\sqrt{T}}\right),
\end{multline}
where $\mathcal{R}_0$ and $\mathcal{R}_{opt} $ are the initial and optimal  $\mathcal{R}(\cdot)$.
\end{theorem}
\begin{proof}
For presentation simplicity, we omit notation $\alpha$ in $\mathcal{R}_t(\cdot)$. By Assumption~\ref{assump:refresh-smoothness} and considering update with sampled mini-batch, i.e., $\thetab^{t+1} \leftarrow \thetab^t - \eta^t\nabla \hat{\mathcal R}_t(\thetab^t)$, we have
\begin{multline}
\mathcal{R}_t(\thetab^{t+1})\le\mathcal R_t(\thetab^t )-\eta^t\langle \nabla\mathcal R_t(\thetab^t),\nabla\hat{\mathcal R}_t(\thetab)\rangle\\
+\frac{\beta_R(\alpha)\eta_t^2}{2}\|\nabla\hat{\mathcal R}_t(\thetab^t)\|^2.
\end{multline}
Taking conditional expectation and using Assumption~\ref{ass:sid-minibatch} and $\eta_t\leq 1/\beta_{R}(\alpha)$, we have 
\begin{equation}\label{eq:rtheta1}
\mathbb E[\mathcal R_t(\thetab^{t+1})]\le \mathbb E[\mathcal R_t(\thetab^t)]-\frac{\eta_t}{2}\|\nabla\mathcal R_t(\thetab^t)\|^2+\frac{\beta_R(\alpha)\eta_t^2\sigma(\alpha)^2}{2B}.
\end{equation}
Based on the definition of $\delta_t^{\mathrm{ref}}(\alpha)$, $\mathcal R_t(\thetab^{t+1})\ge \mathcal R_{t+1}(\thetab^{t+1})- \delta_t^{\mathrm{ref}}(\alpha)$. By substituting $\mathcal R_t(\thetab^{t+1})$ into \eqref{eq:rtheta1} and rearranging, 
\begin{multline}
\frac{\eta_t}{2}\mathbb E\|\nabla\mathcal R_t(\thetab^t)\|^2\le \mathbb E[\mathcal R_t(\thetab^t)]-\mathbb E[\mathcal R_{t+1}(\thetab^{t+1})]\\ +\frac{\beta_R(\alpha)\eta_t^2\sigma(\alpha)^2}{2B}+\mathbb E[\delta_t^{\mathrm{ref}}(\alpha)].
\end{multline}
Summing from $t=0$ to $T-1$ and  based on Lemma \ref{prop:bounded-drift} and the definition of $\mathcal{R}_{opt} $, the proof is complete.
\end{proof}

Theorem \ref{thm:stationarity} shows that ZeroLock  algorithm has a convergence rate of $\tilde{\mathcal{O}}(1/\sqrt{T})$, which differs from the convergence rate  ${\mathcal{O}}(1/\sqrt{T})$ of conventional BP by only polylogarithmic factor. Meanwhile, the study in this section provides the first analytical framework for the local objective construction-based approach under general model chunk division. 


\section{System Design}\label{sec:system}
We propose a pipeline training system for ZeroLock. It is applicable to both multi-GPU server and multi-device scenarios. We first present the overview and detail the system design. Then, we discuss the specific design for parallelism at mobile devices.

\begin{figure}
    \centering
    \includegraphics[height=4.1cm]{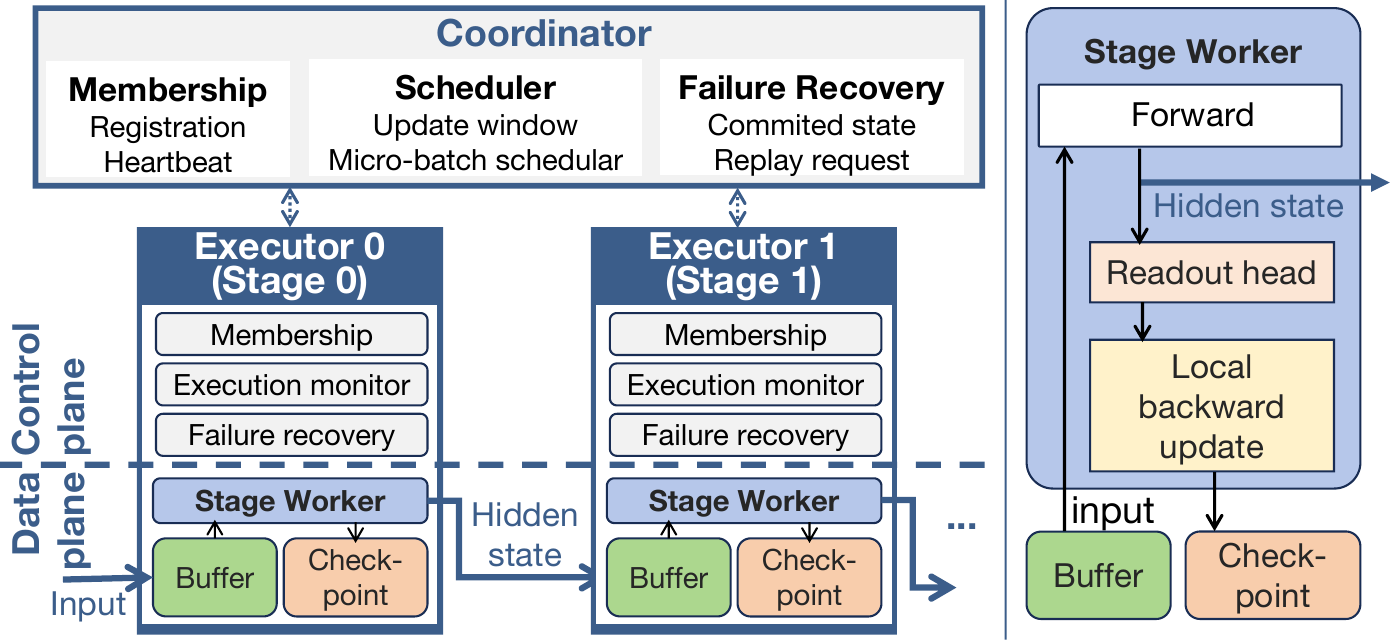}
    \caption{System overview with one coordinator and multiple executors, each corresponding to a stage (i.e., the update of a model chunk).} 
    \label{fig:system}
\end{figure}

\subsection{System Overview}\label{subsec:overview}
As shown in Fig. \ref{fig:system}, this system contains one coordinator and multiple executors (i.e., GPUs or devices), and is separated into control plane and data plane. The  {coordinator} falls in the control plane and has three components: \emph{membership} for  tracking active executors;  \emph{scheduler} for managing the scheduling of micro-batch updates across stages; \emph{failure recovery} for tracking the state of committed requests and assigning replay requests when failure detected. In the \emph{scheduler}, update window defines a contiguous sequence of micro-batches that executors are allowed to and must handle before advancing its process and serving as the unit of checkpoint; micro-batch scheduler indicates the sequence of micro-batches in the update window and will assist  with the micro-batch replay after failure detected.

Each {executor} contains both control plane and data plane. The control plane contains three components:  \emph{membership} for reporting its registration and hearbeat; \emph{execution monitor} for monitoring its execution status; \emph{failure recovery} for execution replay. The data plane at an executor is responsible for handing a stage, i.e., the update of a model chunk. It retrieves either the  model input or the hidden state of the upstream stage from its buffer, performs forward pass to determine its hidden state, transfers the hidden state to the buffer of the downstream stage, and performs a backward update based on its local objective.  
%

We propose the following techniques to achieve lightweight, high-throughput, and robust implementation. \textbf{(I) Early Forwarding.} The forwarding of the hidden state is designed to perform ahead of the local updates, avoiding unnecessary inter-stage waiting. \textbf{(II) Independent Execution and Checkpoint.} Each executor maintains only its own trainable parameters, optimizer state, and checkpoint, enabling recovery at the stage granularity, and  independently  constructs its local signal and perform local updates. \textbf{(III) State-Only Inter-Stage Exchange.} Executors  exchange only the hidden state (but not the gradients or optimizer states) from other stages.  \textbf{(IV) Buffer-Assisted State Exchange.}  Each executor maintains a bounded hidden state buffer, storing the hidden states of its upstream stage for iterations, enabling replay after failure.  Overall, (I)-(III) contribute to lightweight and high-throughput; (II) and (IV) contribute to failure recovery capability. 

\subsection{System Design Details}\label{subsec:detail}
We first present the execution interfaces. Then, we introduce  inter-stage runtime, in-stage execution, and failure recovery.  




\subsubsection{Stage-Execution-Related Interface}
The following defines the unified interfaces related to stage execution at the executors.

\textbf{Input Interface.}  For stage 0 executor, it inputs either token IDs or pre-computed  initial hidden representation (constructed via a local input embedding layer). For each of the subsequent stages, it inputs the hidden states of the upstream stage from the hidden state buffer, filled by its upstream executor. Meanwhile, all stages input the attention mask, position IDs, and labels associated with a specific micro-batch to facilitate local training.

\textbf{Forward-and-Output Interface.} The executor performs forward pass using its chunk to produce its  hidden state, which is then detached from the computation graph and written to the outbound buffer for sending to the downstream stage. 

\textbf{Local Update Interface.} The executor utilizes the input hidden states and labels to construct a local loss and performs a backward pass and optimizer step to update its trainable parameters. This interface supports two ways to construct loss:

(a) Static Model Head. It reuses the pre-trained model's final normalization layer and frozen model head as the readout head, and the token-level cross entropy and task-specific labels are used to construct local loss, as in Section \ref{subsec:unlok}. This design  does not induce new trainable parameters, while it assumes the intermediate hidden states can be interpretable by the final head, which may not hold for complex generative tasks.

(b) Readout Adapter. To bridge the gap between the intermediate hidden states and the frozen model head's readout space, executor can insert a stage-wise readout adapter (i.e., a small residual MLP initialized approximate the identity mapping). During training, this adapter aligns the intermediate hidden space to be more amenable to the frozen model head. 

Besides (a) and (b), this interface also accommodates other lightweight head, e.g., bottleneck projection head or restricted vocabulary head, to reduce computation and memory overhead.


\subsubsection{Inter-Stage Runtime and In-Stage Execution}

The inter-stage runtime (acting as a scheduling engine) organizes multiple micro-batches into an update window at the coordinator. These micro-batches enter the stage pipeline sequentially. For each micro-batch, the stage worker performs \emph{in-stage execution}:
\begin{itemize}
    \item[S1] Input: Retrieve hidden states of the upstream stage from the hidden state buffer; load the attention mask, position IDs, and labels related to the micro-batch. 
    \item[S2] Forward-and-Output: Perform forward pass to produce its hidden state; output it to the outbound buffer.
    \item[S3] Local Training: Backpropagate the local loss and accumulate gradients for the current update window. The optimizer step is applied after the final micro-batch in that window. 
\end{itemize}

Note that downstream execution can start immediately after S2 without waiting for the upstream's local backward in S3, leading to a higher degree of parallelism. 


Across stages, \emph{inter-stage runtime} is responsible for managing the hidden state flow.  The management is in the form of an \emph{entry}, which contains the hidden state tensor of a micro-batch and the related metadata for sending, receiving, and status. The runtime  tracks each entry's transition from the outbound buffer of upstream stage to the hidden state buffer of downstream stage, allowing the coordinator to accurately control the buffer occupancy of each stage without tracking each tensor. Meanwhile, we introduce an in-flight depth to restrict the number of entries between each pair of stages, preventing upstream stages from overwhelming downstream stages.  Furthermore, on GPU, preposted receive is allowed, such that downstream stage can submit $\operatorname{receive\_entry}$ request in advance. Each entry binds to a readiness event. Compute stream waits only for this specific  event, avoiding unrelated communication operations from blocking the compute stream.

\subsubsection{Failure Recovery}\label{subsec:fault}
The failure recovery capability comes from two facts: each stage checkpoints its state independently; the hidden state buffer introduces resilience to failure.  Consequently, after a failure, only the failed stages need to roll back, which retrieves hidden state from the buffer and replaying the requests. 
The brief failure recovery mechanism is as follows:
\begin{itemize}
    \item Proactive Checkpointing:  Each executor periodically captures local checkpoints, including trainable parameters, optimizer states, and progress metadata. 
    \item Reactive Recovery: Upon detecting a failure, the runtime restores the failed stage from its latest checkpoint and replays necessary update windows to catch up to the pipeline frontier. Successful stages retain their execution.
\end{itemize}

\subsection{Mobile Execution Backend}\label{subsec:mobile}

For deployment at mobile devices, we use the  ExecuTorch framework, Meta’s official PyTorch-native edge runtime. It performs  ahead-of-time (AOT) compilation to produce a static computational representation, serialized as a $\texttt{.pte}$ program, which enables deterministic execution on resource-constrained end devices without a Python interpreter. However, ExecuTorch bundles forward and backward computations into a single Python ExecuTorch (PTE) method, where the forward hidden state is exposed only after the local backward gradient computation. This prohibit the early-forwarding opportunity. 

To address this, we insert a lightweight pipeline marker operator within the PTE method—specifically after the detached hidden state output but before the parameter-gradient subgraph. This operator performs no tensor computation; instead, it signals the runtime to capture the hidden outputs and suspend the PTE method while preserving its execution state (stack, tensors, optimizer buffers). We extend the ExecuTorch training runtime with a two-phase Android interface. The first phase executes up to the marker and returns the hidden state, which is immediately transferred to the downstream device. The second phase resumes the same method to complete the local backward pass, after which an on-device AdamW optimizer updates the stage's LoRA parameters. In this case, the transfer to the downstream device overlaps with the local backward computation, making the resulting stage program support local training and early-forwarding without incurring cross-stage gradient Remote Procedure Call (RPC) overhead.



\section{Experiments}\label{sec:experiments}

We build prototype systems using both multi-GPU/CPU server and Android devices respectively. The evaluation on multi-GPU/CPU server shows  (E1) the method comparison on memory and throughput and (E2) failure recovery. The evaluation on Android devices shows (E3) on-device performance.  
 

\subsection{Experimental Settings}
\textbf{Model and data.} We  use TinyLlama as the pre-trained model. Its Transformer layers are partitioned into three consecutive chunks, each assigned to one NVIDIA L40 GPU or one Andriod phone. We apply LoRA  with rank $4$ and scaling factor $16$. On default, experiments are conducted on a fixed 10,000-example subset of AG News; all sequences are padded or truncated to 128 tokens.  Let $b$ denote the number of physical batches processed by one microbatch call; $m$ denotes the number of such microbatch calls before an optimizer update; thus,  batch $B=b m$.  We repeat the experiment with three random seeds. 



\textbf{Methods.} We compare ZeroLock with three approaches. 
(i) GPipe\cite{huang2019gpipe} completes all forward passes in a logical batch before the backward pass. (ii)  1F1B \cite{narayanan2019pipedream} interleaves forward and backward within a logical batch and flushes the pipeline before the shared optimizer update. (iii) PipeDream
 \cite{narayanan2019pipedream} maintains 1F1B execution across update windows via weight stashing (i.e., storing multiple copies), ensuring forward and backward passes use consistent parameter versions.\footnote{Although existing works \cite{guo2024ppll,aljahdali2025fluidpipe,ho2026scpl,shi2026lopt} considered modular decoupling in pipeline parallelism, \cite{ho2026scpl} is not for LLM fine-tuning; \cite{guo2024ppll,aljahdali2025fluidpipe} do not provide open-sourced code; the open-sourced code of \cite{shi2026lopt} does not implement system-level pipeline parallelism (e.g., it places all chunks of a certain batch in one GPU), so comparing it with those deployed in multiple GPUs may not be fair.} 



\subsection{E1:  Memory and Throughput}
Despite the modular update decoupling,  ZeroLock has a comparable accuracy and negative log-likelyhood (NLL) to baseline approaches (see Fig. \ref{fig:e1-curves}). Build upon this, we show its memory reduction and throughput improvement as follows. 
 
\begin{figure}[t]
  \centering
  \includegraphics[width=0.49\columnwidth]{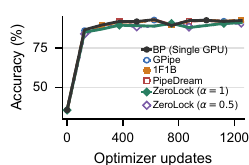}\hfill
  \includegraphics[width=0.49\columnwidth]{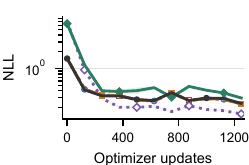}\vspace{-1.0mm}
  \caption{(a) Accuracy and (b) negative-likelihood loss.}
  \label{fig:e1-curves}
\end{figure}

\begin{figure}[t]
  \centering
  
  \includegraphics[width=0.48\textwidth]{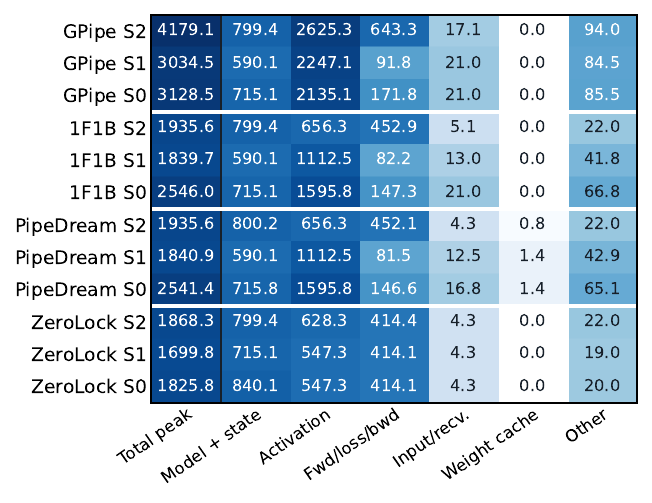}\vspace{-6mm}
  \caption{Peak GPU memory usage and the memory usage of each component under the peak ($b=8$ and $m=4$).}
  \label{fig:memory-map}
\end{figure}

\subsubsection{Memory}
As shown in Fig. \ref{fig:memory-map}, when compared with GPipe,  1F1B, and PipeDream, ZeroLock reduces the mean per-stage peak memory by 47.8\%, 14.7\%, and 14.6\%, respectively, and by 55.3\%, 26.6\%, and 26.5\% for the maximum stage usage, respectively. Most of the reduction comes from the elimination of activations, where the reduction are 75.4\%, 40.8\%, and 48.8\% on per-stage average, respectively. This reduction results from the modular update decoupling of ZeroLock; in comparison, the other approaches implement end-to-end backpropagation, which need to retain activations for backward gradient computation. 
Meanwhile, given $m=4$, GPipe, 1F1B, PipeDream, and ZeroLock experience out-of-memory when $b=10,~20,~20,~28$, respectively, showing that the system with ZeroLock can afford a larger batch size.

Figs. \ref{fig:memory} (a) and (c) show that ZeroLock is more beneficial when the  physical batch to microbatch ratio (i.e., $b/m$) is larger. Specifically, more  physical batches (in a microbatch call) indicates larger forward graphs retained by baseline approaches, leading to more obvious benefits for modular decoupling.

\begin{figure}[t]
    \centering
    \includegraphics[width=0.24\textwidth]{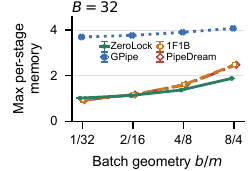}
\includegraphics[width=0.24\textwidth]{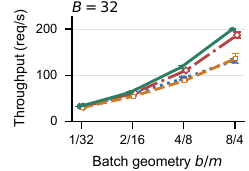}\vspace{-1mm}\\
\quad(a)\quad\qquad\qquad\qquad\qquad\qquad(b)\\
    \includegraphics[width=0.24\textwidth]{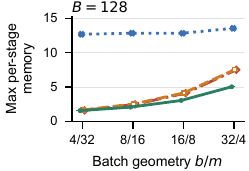}
\includegraphics[width=0.24\textwidth]{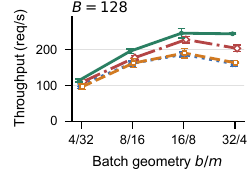}\vspace{-1mm}\\
\quad(c)\quad\qquad\qquad\qquad\qquad\qquad(d)\vspace{-2mm}
\caption{(a) and (c), maximum stage memory usage under $B = 32$ and $B = 128$, respectively;  (b) and (d), throughput under $B = 32$ and $B = 128$, respectively. When $b/m=8/4$, when compared with  PipeDream, ZeroLock reduces the memory by 26.5\% and improves throughput by 4.9\%.}\label{fig:memory}
\end{figure}

\begin{figure}[t]
    \centering
 \includegraphics[width=0.24\textwidth]{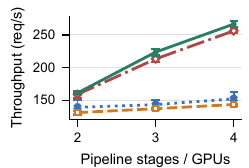}
       \includegraphics[width=0.24\textwidth]{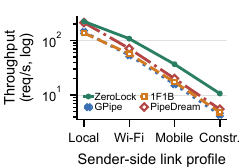}\vspace{-1mm}\\
\quad(a)\quad\qquad\qquad\qquad\qquad\qquad(b) \vspace{-2mm}
  \caption{Throughput under diverse scenarios: (a) number of stages, each executed by a GPU; (b)  diverse sender-side links.
  }\label{fig:throughput}
\end{figure}

\begin{figure}[t]
  \centering
  \includegraphics[width=\columnwidth]{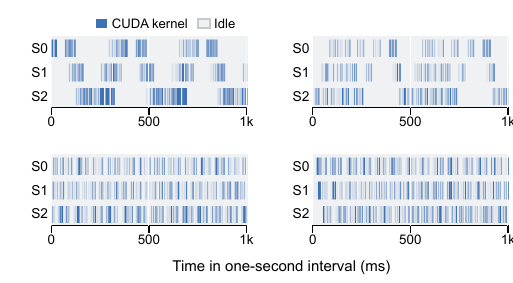}
  \caption{CUDA activity: (a) GPipe; (b) 1F1B; (c) PipeDream; (d) ZeroLock.}
  \label{fig:timeline}
\end{figure}

\subsubsection{Throughput} Figs. \ref{fig:memory} (b) and (d) and Fig. \ref{fig:throughput} show the throughput comparison. The default setting contains three stages, with $b=8$ and $m=4$. First, in Fig. \ref{fig:memory} (b), under the default setting, when compared with GPipe, 1F1B, and PipeDream, ZeroLock improves the throughput by 55.8\%, 62.8\%, and 4.9\%, respectively.  Second, in Figs.  \ref{fig:memory} (b) and (d) and  Fig. \ref{fig:throughput} (a), as the increase of  the physical batches in a microbatch call and stages, ZeroLock and PipeDream demonstrate more significant increase in throughput, showing the better potentials for  reducing bubbles in each call and exploiting additional GPUs, respectively. Finally, Fig. \ref{fig:throughput} (b) shows the throughput under simulated sender-side communication link: Wi-Fi, 2 ms/1000
Mbps (latency/bandwidth); mobile, 10 ms/200 Mbps;  constrained, 30
ms/50 Mbps. ZeroLock is more beneficial when the link is slower. This is because baselines need to send both forward hidden state  and backward hidden gradients, while ZeroLock transfers only forward hidden state. 

Fig. \ref{fig:timeline} visualizes the CUDA activity traces under $b=8$, $m=4$, 512 data samples, and 16 optimizer updates. The plots show a  one-second steady-state interval, during which all three GPUs exhibit compute activity; this interval selection was fixed prior to inspecting method-specific behavior to avoid bias. The Kernel and device copy timestamps are measured externally, without injecting device-side synchronization. As shown in the figure, ZeroLock and PipeDream demonstrate more concurrent stage activities and less GPU idle time. Furthermore, when compared with PipeDream, ZeroLock shows a more balanced activity load across stages and a higher active fraction, with an improvement of 13.7\%. Note that this active fraction serves as a diagnostic indicator of pipeline scheduling rather than a proxy for throughput, as idle periods also encompass host dispatch and Gloo communication overhead.

\begin{table}[t]
  \centering
  \caption{Failure recovery.}
  \label{tab:e5-volatile}
  \small
  \begin{tabular}{lcc}
    \toprule
    Method  & Recovery (ms) & Transfer (MiB) \\
    \midrule
    Synchronous 1F1B  & $2381.2\pm66.6$ & 192 \\
    ZeroLock & $2013.1\pm6.4$ & 96  \\
    \bottomrule
  \end{tabular}
\end{table}

\begin{table}[t]
  \centering
  \caption{On-device evaluation with Andriod phones.}
  \label{tab:e3-devices}
  \setlength{\tabcolsep}{2.5pt}
  \begin{tabular}{lccccc}
    \toprule
    Stage/device & PTE & Queue & Total & PSS  & Temp. \\
    \midrule
    S0/NX809J & 11.43/13.56 & 12.44/17.56 & 38.28/46.90 & 3423.5 & 37.0 \\
    S1/Lenovo L71091 & 7.56/8.32 & 0/0.02 & 15.61/27.51 & 3408.2 & 36.0 \\
    S2/Pixel 10 Pro XL & 9.53/11.40 & 0/5.79 & 10.30/18.99 & 3824.3 & 36.7 \\
    \bottomrule
  \end{tabular}
\end{table}

\begin{figure}[t]
  \centering
\includegraphics[width=0.49\columnwidth]{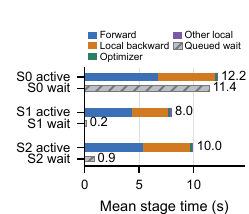}\includegraphics[width=0.49\columnwidth]{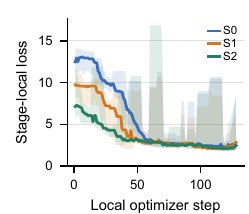}\vspace{-1mm}
\\
\quad(a)\quad\qquad\qquad\qquad\qquad\qquad(b)\vspace{-2mm} 
  \caption{(a) Latency breakdown; (b) Local loss convergence.}
  \label{fig:e3-local-loss}
\end{figure}

\subsection{E2: Failure Recovery}

We evaluate the outage recovery capability of ZeroLock. It is compared with 1F1B. Specifically, 1F1B recovers via global rollback and full replay. In contrast, ZeroLock enables local recovery by reusing buffered hidden states and retaining committed previous updates, eliminating redundant recomputation.  

We consider three stages, and set $b=1$ and $m=8$. The failure setting is as follows: W0--W3, four prelude windows; W4--W7, four failure windows at Stage 1; W8--W11, four  resumed windows. The worker processes and GPU contexts remain alive during the failure.  Recovery latency is defined as the during between  when Stage 1 resumes and when Stage 2 commits W7. In Table \ref{tab:e5-volatile}, ZeroLock reduces the recovery latency by 368 ms and the transfer traffic from 192 to 96 MiB.


\subsection{E3: On-Device Evaluation with Andriod Phones}\label{subsec:android}

For implementation on Android devices, we export three TinyLlama training PTEs containing transformer layers [0,6], [7,13], and [14,21] and deploy them on NX809J, Lenovo L71091, and Pixel 10 Pro XL, respectively. We use 128 data samples for training  with sequence length 128, $b=1$, LoRA rank 8, scale 16, and learning rate $10^{-4}$. Each stage commits 128 optimizer steps and writes nine checkpoints.

 Table \ref{tab:e3-devices} shows the median/95th percentile PTE, queuing, and total duration (in seconds) for a single batch request with the latency break down shown in Fig. \ref{fig:e3-local-loss} (a), the application peak PSS in MiB, and battery temperature in $^\circ C$.  The total wall-clock time of the entire fine-tuning is 1644.1 s with a throughput of 0.0779 records/s. Fig. \ref{fig:e3-local-loss} (b) shows the  loss convergence. These results indicate the implementation is practically feasible.

\section{Conclusion}\label{sec:conclusion}
In this work, we propose a ZeroLock algorithm, which achieves modular update decoupling by local objective construction, for pipeline parallelism in LLM fine-tuning. It effectively removes pipeline bubbles by alleviating the inter-stage waiting and reduces memory usage by mitigating unnecessary activation storage at the algorithm level. We provide the first analytical framework for local objective construction-based approaches under general chunk division and prove that ZeroLock has a convergence rate of $\tilde{\mathcal{O}}(1/\sqrt{T})$, comparable to BP training.  Meanwhile, we design a system for deploying ZeroLock, incorporating techniques such as early forwarding and failure recovery to improve system throughout and robustness.  We build real-world prototype and show that when compared with BP-based baselines, ZeroLock reduces the memory by  26.5\% and improves throughput by  4.9\%. For future direction, it is meaningful to further incorporate operator-level optimization to further improve throughput and reduce memory usage.

\section{Use of AI Disclosure}

ChatGPT was used to assist with code development and experimental setup. Specifically, it was used to configure and troubleshoot the environments required for benchmark reproduction, run and migrate benchmark implementations across different environments, generate Kotlin code for deploying and executing PTE models on mobile devices, and generate batch experiment scripts using torchrun. The AI-generated code and configuration instructions were reviewed, adapted where necessary, tested, and validated by the author. The author takes full responsibility for the final implementation, experimental results, and their interpretation.

\bibliographystyle{IEEEtran}
\bibliography{references_related_work}

\end{document}